\documentclass{article} % For LaTeX2e
\usepackage{nips11submit_e,times}

\usepackage{amsthm,amsfonts,amsmath,amssymb,epsfig,color,float,graphicx,verbatim}
\usepackage{algorithm,algorithmic}

\newtheorem{theorem}{Theorem}

\newtheorem{lemma}{Lemma}

\newcommand{\E}{\mathbb{E}}

\newcommand{\argmax}[1]{\underset{#1}{\mathrm{argmax}}}

\newcommand{\Ocal}{\mathcal{O}}

\newcommand{\Ncal}{\mathcal{N}}

\newcommand{\secref}[1]{Sec.~\ref{#1}}
\newcommand{\subsecref}[1]{Subsection~\ref{#1}}

\renewcommand{\eqref}[1]{Eq.~(\ref{#1})}
\newcommand{\lemref}[1]{Lemma~\ref{#1}}

\newcommand{\thmref}[1]{Thm.~\ref{#1}}

\newcommand{\algref}[1]{Algorithm~\ref{#1}}

\title{From Bandits to Experts: On the Value of Side-Observations}
\author{
    Shie Mannor\\
    Department of Electrical Engineering\\
    Technion, Israel\\
    \texttt{shie@ee.technion.ac.il}
    \And
    Ohad Shamir \\
    Microsoft Research New England\\
    USA\\
    \texttt{ohadsh@microsoft.com}
}

\nipsfinalcopy

\begin{document}

\maketitle

\begin{abstract}
We consider an adversarial online learning setting where a decision maker can choose an action in every stage of the game. In addition to observing the reward of the chosen action, the decision maker gets side observations on the reward he would have obtained had he chosen some of the other actions. The observation structure is encoded as a graph, where node $i$ is linked to
node $j$ if sampling $i$ provides information on the reward of $j$. This setting naturally interpolates between the well-known ``experts'' setting, where the decision maker can view all rewards, and the multi-armed bandits setting, where the decision maker can only view the reward of the chosen action. We develop practical algorithms with provable regret guarantees, which depend on non-trivial graph-theoretic properties of the information feedback structure. We also provide partially-matching lower bounds.
\end{abstract}

\section{Introduction}

One of the most basic learning settings studied in the online learning framework is learning from experts. In its simplest form, we assume that each round $t$, the learning algorithm must choose one of $k$ possible actions, which can be interpreted as following the advice of one of $k$ ``experts''\footnote{The more general setup, which is beyond the scope of this paper, considers $k$ experts providing advice for choosing among $n$ actions, where in general $n\neq k$ \cite{AuerCesFrSc02}.}. At the end of the round, the performance of all actions, measured here in terms of some reward, is revealed. This process is iterated for $T$ rounds, and our goal is to minimize the \emph{regret}, namely the difference between the total reward of the single best action in hindsight, and our own accumulated reward. We follow the standard online learning framework, in which nothing whatsoever can be assumed on the process generating the rewards, and they might even be chosen by an adversary who has full knowledge of our learning algorithm.

A crucial assumption in this setting is that we get to see the rewards of all actions at the end of each round. However, in many real-world scenarios, this assumption is unrealistic. A canonical example is web advertising, where at any timepoint one may choose only a single ad (or small number of ads) to display, and observe whether it was clicked, but not whether other ads would have been clicked or not if presented to the user. This partial information constraint has led to a flourishing literature on multi-armed bandits problems, which model the setting where we can only observe the reward of the action we chose. While this setting has been long studied under stochastic assumptions, the landmark paper \cite{AuerCesFrSc02} showed that this setting can also be dealt with under adversarial conditions, making the setting comparable to the experts setting discussed above. The price in terms of the provable regret is usually an extra $\sqrt{k}$ multiplicative factor in the bound. The intuition for this factor has long been that in the bandit setting, we only get ``$1/k$ of the information'' obtained in the expert setting (as we observe just a single reward rather than $k$). While the bandits setting received much theoretical interest, it has also been criticized for not capturing additional side-information we often have on the rewards of the different actions. This has led to studying richer settings, which make various assumptions on the relationship between the rewards; see below for more details.

In this paper, we formalize and initiate a study on a range of settings that interpolates between the bandits setting and the experts setting.  Intuitively, we assume that after choosing some action $i$, and obtaining the action's reward, we observe not just action $i$'s reward (as in the bandit setting), and not the rewards of all actions (as in the experts setting), but rather some (possibly noisy) information on a \emph{subset} of the other actions. This subset may depend on action $i$ in an arbitrary way, and may change from round to round. This information feedback structure can be modeled as a sequence of directed graphs $G_1,\ldots,G_T$ (one per round $t$), so that an edge from action $i$ to action $j$ implies that by choosing action $i$, ``sufficiently good'' information is revealed on the reward of action $j$ as well. The case of $G_t$ being the complete graph corresponds to the experts setting. The case of $G_t$ being the empty graph corresponds to the bandit setting. The broad scenario of arbitrary graphs in between the two is the focus of our study.

As a motivating example, consider the problem of web advertising mentioned earlier. In the standard multi-armed bandits setting, we assume that we have no information whatsoever on whether undisplayed ads would have been clicked on. However, in many cases, we do have some side-information. For instance, if two ads $i,j$ are for similar vacation packages in Hawaii, and ad $i$ was displayed and clicked on by some user, it is likely that the other ad $j$ would have been clicked on as well. In contrast, if ad $i$ is for running shoes, and ad $j$ is for wheelchair accessories, then a user who clicked on one ad is unlikely to clique on the other. This sort of side-information can be better captured in our setting.

As another motivating example, consider a sensor network  where each sensor collects data from a certain geographic location. Each sensor covers an area that may overlap the area covered by other sensors. At every stage a centralized controller activates one of the sensors and receives input from it. The value of this input is modeled as the integral of some ``information" in the covered area. Since the area covered by  each of the sensors overlaps the area covered by other sensors, the reward obtained when choosing sensor $i$ provides an indication of the reward that would have been obtained when sampling sensor $j$.
% SM: New
A related example comes from ultra wideband communication networks, where every agent can select which channel to use for transmission. When using a channel, the agent senses if the transmission was successful, and also receives some indication of the noise level in other channels that are in adjacent frequency bands \cite{ArsChenBen06}.

Our results portray an interesting picture, with the attainable regret depending on non-trivial properties of these graphs. We provide two practical algorithms with regret guarantees: the ExpBan algorithm that is based on a combination of existing methods, and the more fundamentally novel ELP algorithm that has superior guarantees. We also study lower bounds for our setting. In the case of undirected graphs, we show that the information-theoretically attainable regret is precisely characterized by the average \emph{independence number} (or stability number) of the graph, namely the size of its largest independent set. For the case of directed graphs, we obtain a weaker regret which depends on the average \emph{clique-partition number} of the graphs.
More specifically, our contributions are as follows:

\begin{itemize}
    \item We formally define and initiate a study of the setting that interpolates between learning with  expert advice (with $\Ocal(\sqrt{\log(k)T})$ regret) that assumes that all rewards are revealed and the multi-armed bandits setting (with $\tilde{\Ocal}(\sqrt{kT})$ regret) that assumes that only the reward of the action selected is revealed. We provide an answer to a range of models in between.
    \item The framework we consider assumes that by choosing each action, other than just obtaining that action's reward, we can also observe some side-information about the rewards of other actions. We formalize this as a graph $G_t$ over the actions, where an edge between two actions means that by choosing one action, we can also get a ``sufficiently good'' estimate of the reward of the other action. We consider both the case where $G_t$ changes at each round $t$, as well as the case that $G_t=G$ is fixed throughout all rounds.
    \item We establish upper and lower bounds on the achievable regret, which depends on two combinatorial properties of $G_t$: Its independence number $\alpha(G_t)$ (namely, the largest number of nodes without edges between them), and its clique-partition number $\bar{\chi}(G_t)$ (namely, the smallest number of cliques into which the nodes can be partitioned).
%    \item For general graphs, there is an interesting tradeoff in our upper bounds between regret performance and computational efficiency.
    \item We present two practical algorithms to deal with this setting. The first algorithm, called ExpBan, combines existing algorithms in a natural way, and applies only when $G_t=G$ is fixed at all $T$ rounds. Ignoring computational constraints, the algorithm achieves a regret bound of $\Ocal(\sqrt{\bar{\chi}(G)\log(k)T})$. With computational constraints, its regret bound is $\Ocal(\sqrt{c\log(k)T})$, where $c$ is the size of the minimal clique partition one can efficiently find for $G$. However, note that for general graphs, it is NP-hard to find a clique partition for which $c=\Ocal(k^{1-\epsilon})$ for any $\epsilon>0$.
    \item The second algorithm, called ELP, is an improved algorithm, which can handle graphs which change between rounds. For undirected graphs, where sampling $i$ gives an observation on $j$ and vice versa, it achieves a regret bound of $\Ocal(\sqrt{\log(k) \sum_{t=1}^{T}\alpha(G_t)})$. For directed graphs (where the observation structure is not symmetric), our regret bound is at most $\Ocal(\sqrt{\log(k) \sum_{t=1}^{T}\bar{\chi}(G_t)})$. Moreover, the algorithm is computationally efficient. This is in contrast to the ExpBan algorithm, which in the worst case, cannot efficiently achieve regret significantly better than $\Ocal(\sqrt{k\log(k)T})$.
    \item For the case of a fixed graph $G_t=G$, we present an information-theoretic $\Omega\left(\sqrt{\alpha(G)T}\right)$ lower bound on the regret, which holds regardless of computational efficiency.
    \item We present some simple synthetic experiments, which demonstrate that the potential advantage of the ELP algorithm over other approaches is real, and not just an artifact of our analysis.
\end{itemize}

\subsection{Related Work}
The standard multi-armed bandits problem assumes no relationship between the actions. Quite a few papers studied alternative models, where the actions are endowed with a richer structure. However, in the large majority of such papers, the feedback structure is the same as in the standard multi-armed bandits. Examples include \cite{RusTsi10}, where the actions' rewards are assumed to be drawn from a statistical distribution, with correlations between the actions; and \cite{Agrawal95,KleinSlivUp08}, where the actions reward's are assumed to satisfy some Lipschitz continuity property with respect to a distance measure between the actions.

In terms of other approaches, the combinatorial bandits framework \cite{CesLu09} considers a setting slightly similar to ours, in that one chooses and observes the rewards of some subset of actions. However, it is crucially assumed that the reward obtained is the sum of the rewards of all actions in the subset. In other words, there is no separation between earning a reward and obtaining information on its value. Another relevant approach is partial monitoring, which is a very general framework for online learning under partial feedback. However, this generality comes at the price of tractability for all but specific cases, which do not include our model.

Our work is also somewhat related to the contextual bandit problem (e.g., \cite{LanZhang07,li2010contextual}), where the standard multi-armed bandits setting is augmented with some side-information provided in each round, which can be used to determine which action to pick. While we also consider additional side-information, it is in a more specific sense. Moreover, our goal is still to compete against the best single action, rather than some set of policies which use this side-information.

\section{Problem Setting}

Let $[k]=\{1,\ldots,k\}$ and $[T]=\{1,\ldots,T\}$. We consider a set of actions $1,2,\ldots,k$. Choosing an action $i$ at round $t$ results in receiving a reward $g_{i}(t)$, which we shall assume without loss of generality to be bounded in $[0,1]$. Following the standard adversarial framework, we make no assumptions whatsoever about how the rewards are selected, and they might even be chosen by an adversary. We denote our choice of action at round $t$ as $i_t$. Our goal is to minimize regret with respect to the best single action in hindsight, namely
\[
\max_{i} \sum_{t=1}^{T}g_i(t)-\sum_{t=1}^{T}g_{i_t}(t).
\]
For simplicity, we will focus on a finite-horizon setting (where the number of rounds $T$ is known in advance), on regret bounds which hold in expectation, and on oblivious adversaries, namely that the reward sequence $g_{i}(t)$ is unknown but fixed in advance (see \secref{sec:discussion} for more on this issue).

Each round $t$, the learning algorithm chooses a single action $i_t$. In the standard multi-armed bandits setting, this results in $g_{i_t}(t)$ being revealed to the algorithm, while $g_j(t)$ remains unknown for any $j\neq i_t$. In our setting, we assume that by choosing an action $i$, other than getting $g_i(t)$, we also get some side-observations about the rewards of the other actions. Formally, we assume that one receives $g_i(t)$, and for some fixed parameter $b$ is able to construct unbiased estimates $\hat{g}_j(t)$ for all actions $j$ in some subset of $[k]$, such that $\E[\hat{g}_j(t)|\text{action i  chosen}]=g_j(t)$ and $\Pr(|\hat{g}_j(t)|\leq b)=1$. For any action $j$, we let $N_j(t)$ be the set of actions, for which we can get such an estimate $\hat{g}_j(t)$ on the reward of action $j$. This is essentially the ``neighborhood'' of action $j$, which receives sufficiently good information (as parameterized by $b$) on the reward of action $j$. We note that $j$ is always a member of $N_j$, and moreover, $N_j$ may be larger or smaller depending on the value of $b$ we choose. We assume that $N_j(t)$ for all $j,t$ are known to the learner in advance.

Intuitively, one can think of this setting as a sequence of graphs, one graph per round $t$, which captures the information feedback structure between the actions. Formally, we define $G_t$ to be a graph on the $k$ nodes ${1,\ldots,k}$, with an edge from node $i$ to node $j$ if and only if $j\in N_i(t)$. In the case that $j\in N_i(t)$ if and only if $i\in N_j(t)$, for all $i,j$, we say that $G_t$ is undirected. We will use this graph viewpoint extensively in the remainder of the paper.

\section{The ExpBan Algorithm}\label{sec:simple}

\begin{algorithm}[t]
\caption{The ExpBan Algorithm}
\label{alg:simple}
\begin{algorithmic}
\STATE \textbf{Input:} neighborhood sets $\{N_i(t)\}_{i\in[k]}$.
\STATE Split the graph induced by the neighborhood sets into $c$ cliques ($c\leq k$ as small as possible)
\STATE For each clique, define a ``meta-action'' to be a standard experts algorithm over the actions in the clique
\STATE Run a multi-armed-bandits algorithm over the $c$ meta-actions
\end{algorithmic}
\end{algorithm}

We begin by presenting the ExpBan algorithm (see \algref{alg:simple} above), which builds on existing algorithms to deal with our setting, in the special case where the graph structure remains fixed throughout the rounds - namely, $G_t=G$ for all $t$. The idea of the algorithm is to split the actions into $c$ cliques, such that choosing an action in a clique reveals unbiased estimates of the rewards of all the other actions in the clique. By running a standard experts algorithm  (such as the exponentially weighted forecaster - see \cite[Chapter 2]{CesaBianchiLu06}), we can get low regret with respect to any action in that clique. We then treat each such expert algorithm as a meta-action, and run a standard bandits algorithm (such as the EXP3 \cite{AuerCesFrSc02}) over these $c$ meta-actions. We denote this algorithm as ExpBan, since it combines an experts algorithm with a bandit algorithm.

The following result provides a bound on the expected regret of the algorithm. The proof appears in the appendix.
\begin{theorem}\label{thm:simple}
Suppose $G_t=G$ is fixed for all $T$ rounds. If we run ExpBan using the exponentially weighted forecaster and the EXP3 algorithm, then the expected regret is bounded as follows:\footnote{Using more sophisticated methods, it is now known that the $\log(k)$ factor can be removed (e.g., \cite{AudBub09}). However, we will stick with this slightly less tight analysis for simplicity.}
\begin{equation}\label{eq:simple}
\sum_{t=1}^{T}g_j(t)-\E\left[\sum_{t=1}^{T}g_{i_t}(t)\right] \leq 4b\sqrt{c\log(k)T}.
\end{equation}
For the optimal clique partition, we have $c=\bar{\chi}(G)$, the clique-partition number of $G$.
\end{theorem}

It is easily seen that $\bar{\chi}(G)$ is a number between $1$ and $k$. The case $\bar{\chi}(G)=1$ corresponds to $G$ being a clique, namely, that choosing any action allows us to estimate the rewards of all other actions. This corresponds to the standard experts setting, in which case the algorithm attains the optimal $\Ocal(\sqrt{\log(k)T})$ regret. At the other extreme, $\bar{\chi}(G)=k$ corresponds to $G$ being the empty graph, namely, that choosing any action only reveals the reward of that action. This corresponds to the standard bandit setting, in which case the algorithm attains the standard $\Ocal(\sqrt{\log(k)kT})$ regret. For general graphs, our algorithm interpolates between these regimes, in a way which depends on $\bar{\chi}(G)$.

While being simple and using off-the-shelf components, the ExpBan algorithm has some disadvantages. First of all, for a general graph $G$, it is $NP$-hard to find $c\leq \Ocal(k^{1-\epsilon})$ for any $\epsilon>0$. (This follows from \cite{Zuck07} and the fact that the clique-partition number of $G$ equals the chromatic number of its complement.) Thus, with computational constraints, one cannot hope to obtain a bound better than $\tilde{\Ocal}(\sqrt{kT})$. That being said, we note that this is only a worst-case result, and in practice or for specific classes of graphs, computing a good clique partition might be relatively easy. A second disadvantage of the algorithm is that it is not applicable for an observation structure that changes with time.

\section{The ELP Algorithm}

We now turn to present the ELP algorithm (which stands for ``Exponentially-weighted algorithm with Linear Programming''). Like all multi-armed bandits algorithms, it is based on a tradeoff between exploration and exploitation. However, unlike standard algorithms, the exploration component is not uniform over the actions, but is chosen carefully to reflect the graph structure at each round. In fact, the optimal choice of the exploration requires us to solve a simple linear program, hence the name of the algorithm. Below, we present the pseudo-code as well as a couple of theorems that bound  the expected regret of the algorithm under appropriate parameter choices. The proofs of the theorems appear in the appendix. The first theorem concerns the symmetric observation case, where if choosing action $i$ gives information on action $j$, then choosing action $j$ must also give information on $i$. The second theorem concerns the general case. We note that in both cases the graph $G_t$ may change arbitrarily in time.

\begin{algorithm}
\caption{The ELP Algorithm}
\label{alg:bandits}
\begin{algorithmic}
\STATE \textbf{Input:} $\beta,\{\gamma(t)\}_{t\in[T]},\{s_i(t)\}_{i\in[k],t\in[T]}$, neighborhood sets $\{N_i(t)\}_{i\in[k],t\in[T]}$.
\STATE $\forall~j\in [k]~~~w_j(1):=1/k$.
\FOR{$t=1,\ldots,T$}
    \STATE $\forall~i\in [k]~~~ p_i(t):=(1-\gamma(t))\frac{w_i(t)}{\sum_{l=1}^{k}w_l(k)}+\gamma(t) s_i(t)$
    \STATE Choose action $i_t$ with probability $p_{i_t}(t)$, and receive reward $g_{i_t}(t)$
    \STATE Compute $\hat{g}_{j}(t)$ for all $j\in N_{i_t}(t)$
    \STATE For all $j\in [k]$, let $\tilde{g}_{j}(t)=\frac{\hat{g}_j(t)}{\sum_{l\in N_j(t)}p_l(t)}$ if $i_t\in N_j(t)$, and $\tilde{g}_j(t)=0$ otherwise.
    \STATE $\forall~j\in [k]~~~w_j(t+1)= w_j(t)\exp(\beta \tilde{g}_j(t))$
\ENDFOR
\end{algorithmic}
\end{algorithm}

\subsection{Undirected Graphs}
The following theorem provides a regret bound for the algorithm, as well as appropriate parameter choices, in the case of undirected graphs. Later on, we will discuss the case of directed graphs. In a nutshell, the theorem shows that the regret bound depends on the average independence number $\alpha(G_t)$ of each graph $G_t$ - namely, the size of its largest independent set.
\begin{theorem}\label{thm:main}
Suppose that for all $t$, $G_t$ is an undirected graph. Suppose we run \algref{alg:bandits} using some $\beta\in(0,1/2bk)$, and choosing
\[
\{s_i(t)\}_{i\in [k]}=\argmax{\forall i ~ s_i(t)\geq 0, \sum_{i}s_i(t)=1}~~\min_{j\in[k]}\sum_{l\in N_j(t)}s_l(t),
\]
(which can be easily done via linear programming) and $\gamma(t) = \beta b/\min_{j\in[k]}\sum_{l\in N_j(t)}s_l(t)$.
Then it holds for any fixed action $j$ that
\begin{equation}\label{eq:thmmain1}
\sum_{t=1}^{T}g_j(t)-\E\left[\sum_{t=1}^{T}g_{i_t}(t)\right]
~\leq~
3\beta b^2\sum_{t=1}^{T}\alpha(G_t)
+\frac{\log(k)}{\beta}.
\end{equation}
If we choose $\beta= \sqrt{\log(k)/3b^2\sum_t\alpha(G_t)}$, then the bound equals
\begin{equation}\label{eq:thmmain2}
b\sqrt{3\log(k)\sum_{t=1}^{T}\alpha(G_t)}.
\end{equation}
\end{theorem}

Comparing \thmref{thm:main} with \thmref{thm:simple}, we note that for any graph $G_t$, its independence number $\alpha(G_t)$ lower bounds its clique-partition number $\bar{\chi}(G_t)$. In fact, the gap between them can be very large (see \secref{sec:examples}). Thus, the attainable regret using the ELP algorithm is better than the one attained by the ExpBan algorithm. Moreover, the ELP algorithm is able to deal with time-changing graphs, unlike the ExpBan algorithm.

If we take worst-case computational efficiency into account, things are slightly more involved. For the ELP algorithm, the optimal value of $\beta$, needed to obtain \eqref{eq:thmmain2}, requires knowledge of $\sum_{t=1}^{T}\alpha(G_t)$, but computing or approximating the $\alpha(G_t)$ is NP-hard in the worst case. However, there is a simple fix: we create $\lceil \log(k) \rceil$ copies of the ELP algorithm, where copy $i$ assumes that $\sum_{t=1}^{T}\alpha(G_t)$ equals $2^{i-1}$. Note that one of these values must be wrong by a factor of at most $2$, so the regret of the algorithm using that value would be larger by a factor of at most $2$. Of course, the problem is that we don't know in advance which of those $\lceil \log(k) \rceil$ copies is the best one. But this can be easily solved by treating each such copy as a ``meta-action'', and running a standard multi-armed bandits algorithm (such as EXP3) over these $\lceil \log(k) \rceil$ actions. Note that the same idea was used in the construction of the ExpBan algorithm. Since there are $\lceil \log(k) \rceil$ meta-actions, the additional regret incurred is $\Ocal(\sqrt{\log^2(k)T})$. So up to logarithmic factors in $k$, we get the same regret as if we could actually compute the optimal value of $\beta$.

\subsection{Directed Graphs}

So far, we assumed that the graphs we are dealing with are all undirected. However, a natural extension of this setting is to assume a directed graph, where choosing an action $i$ may give us information on the reward of action $j$, but not vice-versa. It is readily seen that the ExpBan algorithm would still work in this setting, with the same guarantee. For the ELP algorithm, we can provide the following guarantee:

\begin{theorem}\label{thm:maindirected}
Under the conditions of \thmref{thm:main} (with the relaxation that the graphs $G_t$ may be directed), it holds for any fixed action $j$ that
\begin{equation}\label{eq:thmmaindirected1}
\sum_{t=1}^{T}g_j(t)-\E\left[\sum_{t=1}^{T}g_{i_t}(t)\right]
~\leq~
3\beta b^2\sum_{t=1}^{T}\bar{\chi}(G_t),
+\frac{\log(k)}{\beta}.
\end{equation}
where $\bar{\chi}(G_t)$ is the clique-partition number of $G_t$.
If we choose $\beta= \sqrt{log(k)/3b^2\sum_t\bar{\chi}(G_t)}$, then the bound equals
\begin{equation}\label{eq:thmmaindirected2}
b\sqrt{3\log(k)\sum_{t=1}^{T}\bar{\chi}(G_t)}.
\end{equation}
\end{theorem}

Note that this bound is weaker than the one of \thmref{thm:main}, since $\alpha(G_t)\leq \bar{\chi}(G_t)$ as discussed earlier. We do not know whether this bound (relying on the clique-partition number) is tight, but we conjecture that the independence number, which appears to be the key quantity in undirected graphs, is not the correct combinatorial measure for the case of directed graphs\footnote{It is possible to construct examples where the analysis of the ELP algorithm necessarily leads to an $\Ocal(\sqrt{k\log(k)T})$ bound, even when the independence number is $1$}. In any case, we note that even with the weaker bound above, the ELP algorithm still seems superior to the ExpBan algorithm, in the sense that it allows us to deal with time-changing graphs, and that an explicit clique decomposition of the graph is not required. Also, we again have the issue of $\beta$ which is determined by a quantity which is NP-hard to compute, i.e. $\bar{\chi}(G_t)$. However, this can be circumvented using the same trick discussed in the context of undirected graphs.

\section{Lower Bound}

The following theorem provides a lower bound on the regret in terms of the independence number $\alpha(G)$, for a constant graph $G_t=G$.

\begin{theorem}\label{th:lowerbound}
Suppose $G_t = G$ for all $t$, and that actions which are not linked in $G$ get no side-observations whatsoever between them. Then there exists a (randomized) adversary strategy, such that for every $T\ge 374 \alpha(G)^3 $ and any learning strategy, the expected regret is at least $0.06 \sqrt{\alpha(G) T}$.
\end{theorem}
A proof is provided in the appendix. The intuition of the proof is that if the graph $G$ has $\alpha(G)$ independent vertices, then an adversary can make this problem as hard as a standard multi-armed bandits problem, played on $\alpha(G)$ actions. Using a known lower bound of $\Omega(\sqrt{nT})$ for multi-armed bandits on $n$ actions, our result follows\footnote{We note that if the maximal degree of every node is bounded by $d$, it is possible to get the lower bound for $T\ge \Omega(d^2 \alpha(G))$ (as opposed to $T\ge \Omega(\alpha(G)^3)$); see the proof for details.}.

For constant undirected graphs, this lower bound matches the regret upper bound for the ELP algorithm (\thmref{thm:main}) up to logarithmic factors. For directed graphs, the difference between them boils down to the difference between $\bar{\chi}(G)$ and $\alpha(G)$. For many well-behaved graphs, this gap is rather small. However, for general graphs, the difference can be huge - see the next section for details.

\section{Examples}\label{sec:examples}

Here, we briefly discuss some concrete examples of graphs $G$, and show how the regret performance of our algorithms depend on their structure. An interesting issue to notice is the potential gap between the performance of our algorithms, through the graph's independence number $\alpha(G)$ and clique-partition number $\bar{\chi}(G)$.

First, consider the case where there exists a single action, such that choosing it reveals the rewards of all the other actions. In contrast, choosing the other actions only reveal their own reward. At first blush, it may seem that having such a ``super-action'', which reveals everything that happens in the current round, should help us improve our regret. However, the independence number $\alpha(G)$ of such a graph is easily seen to be $k-1$. Based on our lower bound, we see that this ``super-action'' is actually not helpful at all (up to negligible factors).

Second, consider the case where the actions are endowed with some metric distance function, and edge $(i,j)$ is in $G$ if and only if the distance between $i,j$ is at most some fixed constant $r$. We can think of each action $i$ as being in the center of a sphere of radius $r$, such that the reward of action $i$ is propagated to every other action in that sphere. In this case, $\alpha(G)$ is essentially the number of non-overlapping spheres we can \emph{pack} in $G$. In contrast, $\bar{\chi}(G)$ is essentially the number of spheres we need to \emph{cover} $G$. Both numbers shrink rapidly as $r$ increases, improving the regret of our algorithms. However, the sphere covering size can be much larger than the sphere packing size. For example, if the actions are placed as the elements in $\{0,1/2,1\}^n$, we use the $l_{\infty}$ metric, and $r\in (1/2,1)$, it is easily seen that the sphere packing number is just $1$. In contrast, the sphere covering number is at least $2^n=k^{\log_{3}(2)}\approx k^{0.63}$, since we need a separate sphere to cover every element in $\{0,1\}^n$.

Third, consider the random Erd\"{o}s - R\'{e}nyi graph $G=G(k,p)$, which is formed by linking every action $i$ to every action $j$ with probability $p$ independently. It is well known that when $p$ is a constant, the independence number $\alpha(G)$ of this graph is only $\Ocal(\log(k))$, whereas the clique-partition number $\bar{\chi}(G)$ is at least $\Omega(k/\log(k))$. This translates to a regret bound of $\Ocal(\sqrt{kT})$ for the ExpBan algorithm, and only $\Ocal(\sqrt{\log^2(k)T})$ for the ELP algorithm. Such a gap would also hold for a directed random graph.

\section{Empirical Performance Gap between ExpBan and ELP}\label{sec:experiments}

In this section, we show that the gap between the performance of the ExpBan algorithm and the ELP algorithm can be real, and is not just an artifact of our analysis.

To show this, we performed the following simple experiment: we created a random Erd\"{o}s - R\'{e}nyi  graph over $300$ nodes, where each pair of nodes were linked independently with probability $p$. Choosing any action results in observing the rewards of neighboring actions in the graph. The reward of each action at each round was chosen randomly and independently to be $1$ with probability $1/2$ and $0$ with probability $1/2$, except for a single node, whose reward equals $1$ with a higher probability of $3/4$. We then implemented the ExpBan and ELP algorithms in this setting, for $T=30,000$. For comparison, we also implemented the standard EXP3 multi-armed bandits algorithm \cite{AuerCesFrSc02}, which doesn't use any side-observations. All the parameters were set to their theoretically optimal values. The experiment was repeated for varying $p$ and over $10$ independent runs.

The results are displayed in Figure \ref{fig:res}. The $X$-axis is the iteration number, and the $Y$-axis is the mean payoff obtained so far, averaged over the $10$ runs (the variance in the numbers was minuscule, and therefore we do not report confidence intervals). For $p=0.05$, the graph is rather empty, and the advantage of using side observations is not large. As a result, all 3 algorithms perform roughly the same for this choice of $T$. As $p$ increases, the value of side-obervations increase, and the the performance of our two algorithms, which utilize side-observations, improves over the standard multi-armed bandits algorithm. Moreover, for intermediate values of $p$, there is a noticeable gap between the performance of ExpBan and ELP. This is exactly the regime where the gap between the clique-partition number (governing the regret bound of ExpBan) and the independence number (governing the regret bound for the ELP algorithm) tends to be larger as well\footnote{Intuitively, this can be seen by considering the extreme cases - for a complete graph over $k$ nodes, both numbers equal $1$, and for an empty graph over $k$ nodes, both numbers equal $k$. For constant $p\in (0,1)$, there is a real gap between the two, as discussed in \secref{sec:examples}}. Finally, for large $p$, the graph is almost complete, and the advantage of ELP over ExpBan becomes small again (since most actions give information on most other actions).

\begin{figure}[t]
\begin{center}
\includegraphics[trim = 0.1cm 0.45cm 0cm 1.4cm, clip=true, scale=0.57]{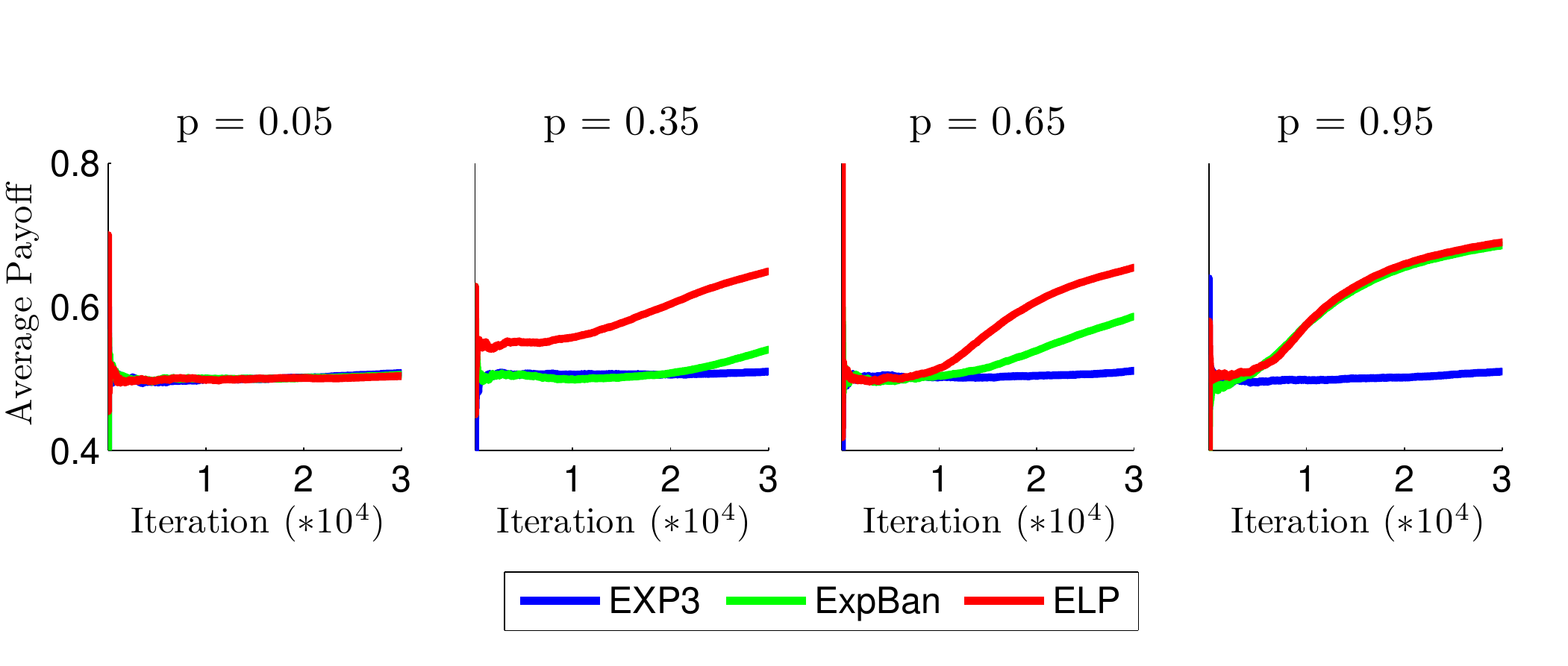}
\end{center}
\vskip -0.3cm
\caption{Experiments on random graphs. \label{fig:res}}
\vskip -0.3cm
\end{figure}

\section{Discussion}\label{sec:discussion}

In this paper, we initiated a study of a large family of online learning problems with side observations. In particular, we studied the broad regime which interpolates between the experts setting and the bandits setting of online learning. We provided algorithms, as well as upper and lower bounds on the attainable regret, with a non-trivial dependence on the information feedback structure.

There are many open questions that warrant further study. First, the upper and lower bounds essentially match only in particular settings (i.e., in undirected graphs, where no side-observations whatsoever, other than those dictated by the graph are allowed). Can this gap be narrowed or closed? Second, our lower bounds depend on a reduction which essentially assumes that the graph is constant over time. We do not have a lower bound for changing graphs. Third, it remains to be seen whether other online learning results can be generalized to our setting, such as learning with respect to policies (as in EXP4 \cite{AuerCesFrSc02}) and obtaining bounds which hold with high probability. Fourth, the model we have studied assumed that the observation structure is known. In many practical cases, the observation structure may be known just partially or approximately. Is it possible to devise algorithms for such cases? %Also, while our results use certain combinatorial graph parameters, there might be other, information-theoretic parameters which more fully capture the essence of the problem in its full generality.

{\bf Acknowledgements.}
This research was supported in part by the Google Inter-university center for Electronic Markets and Auctions.

%\begin{itemize}
%    \item Can the gap between our lower and upper bounds be narrowed or closed? What is the true combinatorial parameter of the graph which determines the regret bound?
%    \item Can the lower bound be improved to time-changing graphs?
 %   \item Is there a real tradeoff between regret performance and computational complexity? Namely, is it provably hard to get the optimal regret?
 %   \item Are there graphs for which our upper bound in \thmref{thm:main} is better (in more than a constant) than the bound which depends on the clique partition number? By how much?
 %   \item Dealing with the case of unknown graph structure?
%\end{itemize}

\newpage

\bibliographystyle{plain}
\bibliography{mybib}

\newpage
\appendix
\section{Proofs}\label{app:proofs}

\subsection{Proof of \thmref{thm:simple}}
Suppose we split the actions into $c$ cliques $C_1,C_2,\ldots,C_c$. First, let us consider the expected regret of the exponentially weighted forecaster ran over any such clique. Denoting the actions of the clique by $1,\ldots,n$, the forecaster works as follows: first, it initializes weights $w_1,\ldots,w_n$ to be $1$. At each round, it picks an action $i$ with probability $w_i/\sum{w_i}$, receives the reward $g_i(t)$, and observes the noisy reward value $\hat{g}_j(t)$ for each of the other actions. It then updates $w_i = w_i \exp(\beta \hat{g}_i(t))$ (for some parameter $\beta\in (0,1/b)$) for all $i=1,\ldots,n$.

The analysis of this algorithm is rather standard, with the main twist being that we only observe unbiased estimates of the rewards, rather than the actual reward. For completeness, we provide this analysis in the following lemma.

\begin{lemma}\label{lem:expp}
The expected regret of the forecaster described above, with respect to the actions in clique $|C_i|$ and under the optimal choice of the parameter $\beta$ is at most $b\sqrt{\log(|C_i|)T}$.
\end{lemma}
\begin{proof}
We define the potential function $W_t = \sum_{j=1}^{n}w_j(t)$, and get that
\[
\frac{W_{t+1}}{W_t} ~\leq~ \sum_{j=1}^{n}\frac{w_j(t)}{\sum_{l=1}^{n}w_l(t)}\exp(\beta \tilde{g}_j(t)).
\]
For notational convenience, let $p_j(t) = \frac{w_j(t)}{\sum_{l=1}^{n}w_l(t)}$. Since $\tilde{g}_j(t)\leq b$, and $\beta\leq 1/b$, we have $\beta\tilde{g}_j(t)\leq 1$. Thus, we can use the inequality $\exp(x)\leq 1+x+x^2$ (which holds for any $x\leq 1$), and get the upper bound
\[
\sum_{j=1}^{n}p_j(t)\left(1+\beta\tilde{g}_j(t)+2\beta^2\tilde{g}_j(t)^2\right)=
1+\beta\sum_{j=1}^{n}\tilde{g}_j(t)+\beta^2\sum_{j=1}^{n}p_j(t)\tilde{g}_j(t)^2.
\]
Taking logarithms and using the fact that $\log(1+x)\leq x$, we get
\[
\log\left(\frac{W_{t+1}}{W_t}\right) ~\leq~ \beta\sum_{j=1}^{n}p_j(t)\tilde{g}_j(t)+\beta^2\sum_{j=1}^{n}p_j(t)\tilde{g}_j(t)^2.
\]
Summing over all $t$, and canceling the resulting telescopic series, we get
\begin{equation}\label{eq:upbound1}
\log\left(\frac{W_{T+1}}{W_1}\right)
~\leq~ \sum_{t=1}^{T}\left(\beta\sum_{j=1}^{n}p_j(t)\tilde{g}_j(t)+\beta^2\sum_{j=1}^{n}p_j(t)\tilde{g}_j(t)^2\right). \end{equation}
Also, for any fixed action $i$, we have
\begin{equation}\label{eq:lowbound1}
\log\left(\frac{W_{T+1}}{W_1}\right) \geq \log\left(\frac{w_i(T+1)}{W_1}\right) = \beta\sum_{t=1}^{T}\tilde{g}_i(t)-\log(n).
\end{equation}
Combining \eqref{eq:upbound1} with \eqref{eq:lowbound1} and rearranging, we get
\[
\sum_{t=1}^{T}\tilde{g}_i(t) - \sum_{t=1}^{T}\sum_{j=1}^{n}p_{j}(t)\tilde{g}_j(t) \leq \frac{\log(n)}{\beta}+\beta\sum_{t=1}^{T}\sum_{j=1}^{n} p_j(t)\tilde{g}_j(t)^2.
\]
Taking expectations on both sides, and using the facts that $\E[\tilde{g}_j(t)] = g_j(t)$ for all $j,t$, and $|\tilde{g}_j(t)|\leq b$ with probability $1$, we get
\[
\sum_{t=1}^{T}g_i(t) - \sum_{t=1}^{T}\sum_{j=1}^{n}p_{j}(t)g_j(t) \leq \frac{\log(n)}{\beta}+\beta b^2 T.
\]
Thus, by picking $\beta=\sqrt{log(n)/b^2 T}$, we get that the expected regret is at most $b\sqrt{\log(n)T}$.
\end{proof}

Now, we define each such forecaster (one per clique $C_i$) as a meta-action, and run the EXP3 algorithm on the $c$ meta-actions. By the standard guarantee for this algorithm (see corollary 3.2 in \cite{AuerCesFrSc02}), the expected regret incurred by that algorithm with respect to any fixed meta-action is at most $3b\sqrt{c\log(c)T}$. Combining this with \lemref{lem:expp}, we get that the total expected regret of the ExpBan algorithm with respect to any single action is at most
\[
\max_i b\sqrt{\log(|C_i|)T}+3b\sqrt{c\log(c)T} \leq b\sqrt{\log(k)T}+3b\sqrt{c\log(k)T},
\]
which is at most $4b\sqrt{\log(k)cT}$ since $c\geq 1$.

\subsection{Proof of \thmref{thm:main}}\label{subsec:proofthmmain}

To prove the theorem, we will need three lemmas. The first one is straightforward and follows from the definition of $\tilde{g}_j(t)$. The second is a key combinatorial inequality. We were unable to find an occurrence of this inequality in any previous literature,  although we are aware of very special cases proven in the context of cyclic sums (see for instance \cite{Baston74}). The third lemma allows us to derive a more explicit bound by examining a particular choice of $\{s_i(t)\}_{i\in[k],t\in[T]}$.

\begin{lemma}\label{lem:momentbounds}
For all fixed $t,j$, we have
\[
\E\left[\tilde{g}_j(t)\right] = g_{j}(t)
\]
as well as
\[
\E\left[\sum_{j=1}^{k}p_j(t)\tilde{g}_j(t)^2\right] \leq b^2\sum_{j=1}^{k}\frac{p_j(t)}{\sum_{l\in N_j(t)}p_l(t)}.
\]

\end{lemma}
\begin{proof}
It holds that
\[
\E\left[\tilde{g}_j^i(t)\right] = \sum_{i=1}^{k}p_{i}(t)\E[\tilde{g}_j(t)~|~\text{action i was picked}] = \sum_{i\in N_j(t)}p_i(t)  \frac{g_{j}(t)}{\sum_{l\in N_j(t)} p_l(t)} = g_{j}(t).
\]
As to the second part, we have
\begin{align*}
&\E\left[\sum_{j=1}^{k}p_j(t)\tilde{g}_j(t)^2\right]
~=~ \sum_{i,j=1}^{k}p_j(t)p_i(t)\E\left[\tilde{g}_j(t)^2~|~ \text{action i was picked} \right]\\
& \leq \sum_{j=1}^{k}\sum_{i\in N_j(t)}p_j(t)p_i(t)\frac{b^2}{\left(\sum_{l\in N_j(t)}p_l(t)\right)^2}
~=~ b^2\sum_{j=1}^{k}\frac{p_j(t)}{\sum_{l\in N_j(t)}p_l(t)}.
\end{align*}
\end{proof}

\begin{lemma}\label{lem:key}
Let $G$ be a graph over $k$ nodes, and let $\alpha(G)$ denote the independence number of $G$ (i.e., the size of its largest independent set). For any $j\in[k]$, define $N_j$ to be the nodes adjacent to node $j$ (including node $j$). Let $p_1,\ldots,p_k$ be arbitrary positive weights assigned to the node. Then it holds that
\[
\sum_{i=1}^{k}\frac{p_i}{\sum_{l\in N_i}p_l} ~\leq~ \alpha(G).
\]
\end{lemma}

\begin{proof}
We will actually prove the claim for any nonnegative weights $p_1,\ldots,p_k$ (i.e., they are allowed to take $0$ values), under the convention that if $p_j=0$ and $\sum_{l\in N_j}p_i=0$ as well, then $\sum_{i=1}^{k}p_i/\sum_{l\in N_i}p_i=1$.

Suppose on the contrary that there exist some values for $p_1,\ldots,p_k$ such that $\sum_{i=1}^{k}p_i/\sum_{l\in N_i}p_i > \alpha(G)$. Now, if $p_1,\ldots,p_k$ are non-zero only on an independent set $S$, then
\[
\sum_{i=1}^{k}\frac{p_i}{\sum_{l\in N_i}p_i} = \sum_{i\in S}\frac{p_i}{p_i} = |S|.
\]
Since $|S|\leq \alpha(G)$, it follows that there exist some adjacent nodes $r,s$ such that $p_r,p_s>0$. However, we will show that in that case, we can only increase the value of $\sum_{i=1}^{k}p_i/\sum_{l\in N_i}p_i$ by shifting the entire weight $p_r+p_s$ to either node $r$ or node $s$, and putting weight $0$ at the other node. By repeating this process, we are guaranteed to eventually arrive at a configuration where the weights are non-zero on an independent set. But we've shown above that in that case, $\sum_{i=1}^{k}p_i/\sum_{l\in N_i}p_i \leq \alpha(G)$, so this means the value of this expression with respect to the original configuration was at most $\alpha(G)$ as well.

To show this, let us fix $p_r+p_s=c$ (so that $p_s=c-p_r$) and consider how the value of the expression changes as we vary $p_r$. The sum in the expression $\sum_{i=1}^{k}p_i/\sum_{l\in N_i}p_i$ can be split to 6 parts: when $i=r$, when $i=s$, when $i$ is a node adjacent to $s$ but not to $r$, when $i$ is  adjacent to $r$ but not to $s$, when $i$ is adjacent to both, and when $i$ is adjacent to neither of them. Decomposing the sum in this way, so that $p_r$ appears everywhere explicitly, we get
\begin{align*}
&\frac{p_r}{c+\sum_{l\in N_r\setminus {r,s}}p_l}
+\frac{c-p_r}{c+\sum_{l\in N_j\setminus {r,s}}p_l}
+\sum_{i:\{r,s\}\cap N_i = s}\frac{p_i}{c-p_r+\sum_{l\in N_i\setminus{s}}p_l}\\
&+\sum_{i:\{r,s\}\cap N_i = r}\frac{p_i}{p_r+\sum_{l\in N_i\setminus{r}}p_l}
+\sum_{i:i\notin\{r,s\},{r,s}\subseteq N_i}\frac{p_i}{c+\sum_{l\in N_i\setminus\{r,s\}}p_l}
+\sum_{i:\{r,s\}\cap N_i=\emptyset}\frac{p_i}{\sum_{l\in N_i}p_l}.
\end{align*}
It is readily seen that each of the $6$ elements in the sum above is convex in $p_r$. This implies that the maximum of this expression is attained at the extremes, namely either $p_r=0$ (hence $p_s=c$) or $p_r=c$ (hence $p_s=0$). This proves that indeed shifting weights between adjacent nodes can only increase the value of $\sum_{i=1}^{k}p_i/\sum_{l\in N_i}p_i$, and as discussed earlier, implies the result stated in the lemma.
\end{proof}

\begin{lemma}\label{lem:maxratio}
Consider a graph $G$ over nodes $1,\ldots,k$, and let $\alpha(G)$ be its independence number. For any $j\in[k]$, define $N_j$ to be the nodes adjacent to node $j$ (including node $j$). Then there exist values of $s_1,\ldots,s_k$ on the $k$-simplex, such that
\begin{equation}
\frac{1}{\min_{j\in[k]}\sum_{l\in N_j}s_l}\leq \alpha(G).\label{eq:maxratio}
\end{equation}
\end{lemma}
\begin{proof}
Let $S$ be a largest independent set of $G$, so that $|S|=\alpha(G)$.
Consider the following specific choice for the values of $s_1,\ldots,s_k$: For any $j$ such that $j\in S$, let $s_j=1/\alpha(G)$, and $s_j=0$ otherwise.
Suppose there was some node $j$ such that $\sum_{l\in N_j} s_l=0$. By the way we chose values for $s_1,\ldots,s_k$, this implies that node $j$ is not adjacent to any node in $S$, so $S\cup \{j\}$ would also be an independent set, contradicting the assumption that $S$ is a largest independent set. But since each value of $s_l$ is either $0$ or $1/\alpha(G)$, it follows that $\sum_{l\in N_j} s_l>1/\alpha(G)$. This is true for any node $j$, from which \eqref{eq:maxratio} follows.
\end{proof}

We now turn to the proof of the theorem itself.

\begin{proof}[Proof of \thmref{thm:main}]
With the key lemmas at hand, most of the remaining proof is rather similar to the standard analysis for multi-armed bandits (e.g., \cite{AuerCesFrSc02}). We define the potential function $W_t = \sum_{j=1}^{k}w_j(t)$, and get that
\begin{equation}\label{eq:begin}
\frac{W_{t+1}}{W_t} ~\leq~ \sum_{j=1}^{k}\frac{w_j(t)}{\sum_{l=1}^{k}w_l(t)}\exp(\beta \tilde{g}_j(t)).
\end{equation}
We have that $\beta\tilde{g}_j(t)\leq 1$, since by definition of $\beta$ and $\tilde{g}_j(t)$,
\[
\beta\tilde{g}_j(t) \leq \frac{\beta b}{\sum_{l\in N_j(t)}p_l(t)}
\leq \frac{\beta b}{\sum_{l\in N_j(t)}\gamma(t)s_l(t)}
= \frac{\beta b}{\sum_{l\in N_j(t)}s_l(t)}\frac{\min_{j\in[k]} \sum_{l\in N_j(t)}s_l(t)}{\beta b} \leq 1.
\]
 Using the definition of $p_j(t)$ and the inequality $\exp(x)\leq 1+x+x^2$ for any $x\leq 1$, we can upper bound \eqref{eq:begin} by
\begin{align*}
&\sum_{j=1}^{k}\frac{p_j(t)-\gamma(t) s_j(t)}{1-\gamma(t)}\left(1+\beta\tilde{g}_j(t)+\beta^2\tilde{g}_j(t)^2\right)\\
&\leq~ 1+\frac{\beta}{1-\gamma(t)}\sum_{j=1}^{k}p_j(t)\tilde{g}_j(t)+\frac{2\beta^2}{1-\gamma(t)}\sum_{j=1}^{k}p_j(t)\tilde{g}_j(t)^2.
\end{align*}
Taking logarithms and using the fact that $\log(1+x)\leq x$, we get
\[
\log\left(\frac{W_{t+1}}{W_t}\right) ~\leq~ \frac{\beta}{1-\gamma(t)}\sum_{j=1}^{k}p_j(t)\tilde{g}_j(t)+\frac{\beta^2}{1-\gamma(t)}\sum_{j=1}^{k}p_j(t)\tilde{g}_j(t)^2.
\]
Summing over all $t$, and canceling the resulting telescopic series, we get
\begin{equation}\label{eq:upbound}
\log\left(\frac{W_{T+1}}{W_1}\right)
~\leq~ \sum_{t=1}^{T}\sum_{j=1}^{k}\frac{\beta}{1-\gamma(t)}p_{j}(t)\tilde{g}_j(t)
+\sum_{t=1}^{T}\sum_{j=1}^{k}\frac{\beta^2}{1-\gamma(t)}p_j(t)\tilde{g}_j(t)^2.
\end{equation}
Also, for any fixed action $i$, we have
\begin{equation}\label{eq:lowbound}
\log\left(\frac{W_{T+1}}{W_1}\right) \geq \log\left(\frac{w_i(T+1)}{W_1}\right) = \beta\sum_{t=1}^{T}\tilde{g}_i(t)-\log(k).
\end{equation}
Combining \eqref{eq:upbound} with \eqref{eq:lowbound} and rearranging, we get
\[
\beta\sum_{t=1}^{T}\tilde{g}_i(t) - \sum_{t=1}^{T}\sum_{j=1}^{k}\frac{\beta}{1-\gamma(t)}p_{j}(t)\tilde{g}_j(t) \leq \log(k)+\sum_{t=1}^{T}\sum_{j=1}^{k}\frac{\beta^2}{1-\gamma(t)}p_j(t)\tilde{g}_j(t)^2.
\]
Taking expectations on both sides, and using \lemref{lem:momentbounds}, we get
\[
\beta\sum_{t=1}^{T}g_i(t)-\sum_{t=1}^{T}\sum_{j=1}^{k}\frac{\beta}{1-\gamma(t)}p_j(t)g_j(t)
~\leq~
\log(k)+\sum_{t=1}^{T}\sum_{j=1}^{k}\frac{b^2\beta^2}{1-\gamma(t)}\frac{p_j(t)}{\sum_{l\in N_j(t)}p_l(t)}.
\]
After some slight manipulations, and using the fact that $g_j(t)\in [0,1]$ for all $j,t$, we get
\[
\sum_{t=1}^{T}g_i(t)-\sum_{t=1}^{T}\sum_{j=1}^{k}p_j(t)g_j(t)
~\leq~
\sum_{t=1}^{T}\gamma(t) + \frac{\log(k)}{\beta}+\sum_{t=1}^{T}\frac{b^2\beta}{1-\gamma(t)}\sum_{j=1}^{k}\frac{p_j(t)}{\sum_{l\in N_j(t)}p_l(t)}.
\]
We note that $1/(1-\gamma(t))$ can be upper bounded by $2$, since by definition of $s_i(t)$,
\[
\gamma(t) = \frac{\beta b}{\max_{a_1,\ldots,a_k}\min_{j\in[k]}\sum_{l\in N_j(t)}a_l(t)}
\leq \frac{\beta b}{\min_{j\in[k]}\sum_{l\in N_j(t)}(1/k)}\leq \beta bk\leq 1/2.
\]
Plugging this in as well as our choice of $\gamma(t)$ in the $\sum_{t}\gamma(t)$ term, and slightly simplifying, we get the upper bound
\begin{equation}\label{eq:actualbound}
\sum_{t=1}^{T}g_i(t)-\sum_{t=1}^{T}\E[g_{i_t}(t)]
~\leq~
\beta b^2 \left( \sum_{t=1}^{T}\frac{1}{\min_{j\in[k]}\sum_{l\in N_j(t)}s_l(t)} +2\sum_{j=1}^{k}\frac{p_j(t)}{\sum_{l\in N_j(t)}p_l(t)}
\right)+\frac{\log(k)}{\beta}.
\end{equation}
Now, we recall that the $\{s_i(t)\}$ terms were chosen so as to minimize the bound above. Thus, we can upper bound it by any fixed choice of $\{s_i(t)\}$. Invoking \lemref{lem:maxratio}, as well as \lemref{lem:key}, the theorem follows.
\end{proof}

\subsection{Proof of \thmref{thm:maindirected}}

The proof is very similar to the one of \thmref{thm:main}, so we'll only point out the differences.

Referring to the proof of \thmref{thm:main} in \subsecref{subsec:proofthmmain}, The analysis is identical up to \eqref{eq:actualbound}. To upper bound the terms there, we can still invoke \lemref{lem:maxratio}. However, \lemref{lem:key}, which was used to upper bound $\sum_{j=1}^{k}p_j(t)/\sum_{l\in N_j(t)}p_l(t)$, not longer applies (in fact, one can show specific counter-examples). Thus, in lieu of \lemref{lem:key}, we will opt for the following weaker bound: Let $C_1,\ldots, C_{\bar{\chi}(G_t)}$ be a smallest possible clique partition of $G_t$. Then we have
\[
\sum_{i=1}^{\bar{\chi}(G_t)}\sum_{j\in C_{i}}\frac{p_j(t)}{\sum_{l\in N_j(t)}p_l(t)} \leq \sum_{i=1}^{\bar{\chi}(G_t)}\sum_{j\in C_{i}}\frac{p_j(t)}{\sum_{l\in C_{i}}p_l(t)} = \bar{\chi}(G_t).
\]
Plugging this upper bound as well as \lemref{lem:maxratio} into \eqref{eq:actualbound}, and using the fact that $\alpha(G_t)\leq \bar{\chi}(G_t)$ for any graph $G_t$, the result follows.

\subsection{Proof of Theorem \ref{th:lowerbound}}

Suppose that we are given a graph $G$ with an independence number $\alpha(G)$. Let $\mathcal{N}$ denote an independent set of $\alpha(G)$ nodes (i.e., no two nodes are connected). Suppose we have an algorithm $\mathcal{A}$ with a low expected regret for every sequence of rewards. We will use this algorithm to form an algorithm for the standard multi-armed bandits problem (with no-side observations). We will then resort to the known lower bound for this problem, to get a lower bound for our setting as well.

Consider first a standard multi-armed bandits game on $\alpha(G)$ actions (with no side-observations), with the following randomized strategy for the adversary: the adversary picks one of the $\alpha(G)$ actions uniformly at random, and at each round, assigns it a random Bernoulli reward with parameter $1/2+\epsilon$ (where $\epsilon$ will be specified later). The other actions are assigned a random Bernoulli reward with parameter $1/2$. Roughly speaking, Theorem 6.11 of \cite{CesaBianchiLu06} shows that with this strategy and for $\epsilon=\Theta(\sqrt{\alpha(G)/T})$, the expected regret of any learning algorithm is at least $\Omega(\sqrt{\alpha(G)T})$.

Now, suppose that for the setting with side-observations, played over the graph $G$, there exists a learning strategy $\mathcal{A}$ that achieves expected cumulative regret of at most $R_{\mathcal{A}}(T)$, for the graph $G$ over $T$ rounds, with respect to any adversary strategy. We will now show how to use $\mathcal{A}$ for the standard multi-armed bandits game described above. To that end, arbitrarily assign the $\alpha(G)$ actions to the $\alpha(G)$ independent nodes in $\Ncal$. We will then implement the following strategy $\mathcal{A}'$:  whenever $\mathcal{A}$ chooses one of the actions in $\mathcal{N}$, we choose the corresponding action in the multi-armed bandits problem and feed the reward back to $\mathcal{A}$ (the reward of all neighboring nodes is 0, which we feed back to $\mathcal{A}$ as well). Whenever $\mathcal{A}$ chooses a node $j$ not in $\mathcal{N}$, we use the next $|N_j\cap\mathcal{N}|$ rounds (where $N_j$ is the neighborhood set of $j$) to do ``pure exploration:'' we go over all the neighbors of node $j$ that belong to $\mathcal{N}$ in some fixed order, and choose each of them once (since rewards are assumed stochastic the order does not matter). Nodes in $N_j\setminus\mathcal{N}$ are known to yield a reward of $0$. The rewards of node $j$ and all its neighbors are then fed to $\mathcal{A}$, as if they were side observations obtained in a single round by choosing a node not in $\mathcal{N}$. Since the rewards are chosen i.i.d., the distribution of these rewards is identical to the case where $\mathcal{A}$ was really implemented with side-observations. We denote $R_{\mathcal{A}'}(T)$ as the expected regret of this strategy $\mathcal{A}'$, after $T$ rounds.

We make the following observation: suppose $\mathcal{A}$ achieves an expected regret satisfying
$$
R_\mathcal{A}(T) \le \sqrt{\alpha(G)  T}
$$
(we can assume this since our goal is to provide a lower bound which will only be smaller). Then the number of times  $\mathcal{A}$ chose actions outside $\mathcal{N}$ must be smaller than  $2\sqrt{\alpha(G)T}$. This is because whenever $\mathcal{A}$ chooses an action not in $\mathcal{N}$ it receives a reward of 0 while the highest expected reward is bigger than $1/2$, so the expected per-round regret would increase by at least $1/2$.

We apply $\mathcal{A}'$ at each round, till $\mathcal{A}$ is called $T$ times. Let $T'$ be the (possibly random) number of rounds which elapsed. It holds that $T'\ge T$, since we have the $T'-T$ pure exploration rounds where $\mathcal{A}$ is not called. In these exploration rounds, we pull arms in $\mathcal{N}$, so our expected regret in those rounds is at most $\epsilon$. Moreover, by the observation above, the number of such rounds is at most $2\alpha(G)\sqrt{\alpha(G)T}$, since $\mathcal{A}$ may choose an action outside $\mathcal{N}$ at most $2\sqrt{\alpha(G)T}$ times, and this follows by at most $|\mathcal{N}|=\alpha(G)$ pure exploration steps. In rounds where we do not do exploration steps, the expected per-round regret of $\mathcal{A'}$ is the same as the expected per-round regret of $\mathcal{A}$. Overall, this implies that
\begin{equation}\label{eq:regcomp}
R_{\mathcal{A}'}(T') \leq R_\mathcal{A}(T) + 2\epsilon \alpha(G) \sqrt{\alpha(G)T}
\end{equation}
Since the expected regret is monotone in the number of rounds, we can lower bound $R_{\mathcal{A}'}(T')$ by $R_{\mathcal{A}'}(T)$. Rearranging, we get
$$
R_\mathcal{A}(T) \ge R_{\mathcal{A}'}(T) - 2\epsilon \alpha(G) \sqrt{\alpha(G)  T}.
$$
Now, $\mathcal{A}'$ is a strategy for the standard multi-armed bandits setting, with a randomized adversary strategy which is identical to the one used to establish the lower bound of \cite[Theorem 6.11]{CesaBianchiLu06}. Using this lower bound, by selecting $\epsilon = \sqrt{c_1 \alpha(G) /T}$ with $c_1=1/(8 \ln(4/3))$, we obtain
\begin{equation}\label{eq:regcomp2}
R_\mathcal{A}(T) \ge \sqrt{T\alpha(G)} c_2 - 2\sqrt{c_1} \alpha(G) ^2,
\end{equation}
where the first term of the right hand side comes from Page 168 in  \cite{CesaBianchiLu06} and
$$ c_2 = \frac{\sqrt{2}-1}{\sqrt{32 \ln(4/3)}}.$$
Since $T\ge 16 \alpha(G)^3 c_1/c_2^2 $, we have that $R_\mathcal{A}(T) \ge  \sqrt{T\alpha(G)} c_2/2$. Plugging in the values of $c_1,c_2$ above, the result follows.

Finally, we note that if the maximal degree of any node in $G$ is bounded by $d$, then \eqref{eq:regcomp} can be improved to
\[
R_{\mathcal{A}'}(T') \leq R_\mathcal{A}(T) + 2\epsilon d \sqrt{\alpha(G)T},
\]
since the number of pure-exploration steps following a call to $\mathcal{A}$ is at most $d$ rather than $\alpha(G)$. Repeating the analysis above, we get that \eqref{eq:regcomp2} is replaced by
\[
R_\mathcal{A}(T) \ge \sqrt{T\alpha(G)} c_2 - 2\sqrt{c_1} d \alpha(G).
\]
This allows us to give the same lower bound, for any $T\ge 16 \alpha(G)d^2 c_1/c_2^2 $, as opposed to $T\ge 16 \alpha(G)^3 c_1/c_2^2 $ as before.

\end{document}